\documentclass{hld2025} %

\usepackage{microtype}
\usepackage{graphicx}
\usepackage{booktabs}

\usepackage{subcaption}
\usepackage{float}

\usepackage{hyperref}

\definecolor{ForestGreen}{RGB}{34,139,34}

\newcommand{\ODE}{{}}%

\newcommand{\T}{\mathsf{T}}

\def\EE{{\mathbb{E}}}

\numberwithin{equation}{section}

\usepackage{amsmath,amsfonts,bm}

\def\eqref#1{equation~\ref{#1}}

\def\1{\bm{1}}

\def\mI{{\bm{I}}}

\DeclareMathAlphabet{\mathsfit}{\encodingdefault}{\sfdefault}{m}{sl}
\SetMathAlphabet{\mathsfit}{bold}{\encodingdefault}{\sfdefault}{bx}{n}

\def\[{\left[}
\def\]{\right]}
\def\({\left(}
\def\){\right)}

\newcommand{\Sh}{\hat{\Sigma}}
\newcommand{\Rh}{\hat{R}}
\newcommand{\df}{d\text{f}}
\newcommand{\Tr}{\text{Tr}}

\usepackage{pifont}

\title[A solvable generative model with a linear, one-step denoiser]{A solvable generative model with a linear, one-step denoiser}

\hldauthor{%
\Name{Indranil Halder} \Email{ihalder@g.harvard.edu}\\
\addr {Harvard University, Cambridge, MA, USA} 
}

\begin{document}

\maketitle

\begin{abstract}%
We develop an analytically tractable single-step diffusion model based on a linear denoiser and present an explicit formula for the Kullback-Leibler divergence between the generated and sampling distribution, taken to be isotropic Gaussian, showing the effect of finite diffusion time and noise scale. Our study further reveals that the monotonic fall phase of Kullback-Leibler divergence begins when the training dataset size reaches the dimension of the data points.  Finally, for large-scale practical diffusion models, we explain why a higher number of diffusion steps enhances production quality based on the theoretical arguments presented before. 
\end{abstract}

\section{Introduction}
In recent years, generative artificial intelligence has made tremendous advancements - be it image, audio, video, or text domains—on an unprecedented scale.  Diffusion models \cite{sohldickstein2015deepunsupervisedlearningusing, song2019generative, ho2020denoising, kadkhodaie2020solving, song2021denoising} are among the most successful frameworks \citep{ramesh2022hierarchicaltextconditionalimagegeneration, rombach2022highresolutionimagesynthesislatent, saharia2022photorealistictexttoimagediffusionmodels, sora}. The quality of the generated images can be enhanced through guided diffusion at the cost of reduced diversity \citep{dhariwal2021diffusionmodelsbeatgans,ho2022classifierfreediffusionguidance, wu2024theoreticalinsightsdiffusionguidance, bradley2024classifierfreeguidancepredictorcorrector, chidambaram2024doesguidancedofinegrained}. 
Also, experimentally it is observed that increasing diffusion steps leads to more visually appealing images \cite{karras2022elucidatingdesignspacediffusionbased}.
 Theoretically understanding this phenomena and generalization ability \cite{yoon2023diffusion, zhang2023emergence, kadkhodaie2023generalization,carlini2023extracting, somepalli2022diffusion, favero2025compositional, okawa2023compositional,park2024emergence, kamb2024analytic,li2024critical} of the diffusion models is a challenging task.
Keeping this goal in mind, we introduce and study a linear denoiser based generative model that is analytically tractable and features some of the properties of a single denoising step in a realistic diffusion model.

\subsection{Our contributions}
Our main contributions to this paper are as follows:
\begin{enumerate}
    \item We define a linear denoiser based generative model. Within the framework of the model, we present explicit formula for the Kullback-Leibler divergence between generated and sampling distribution, taken to be isotropic Gaussian, showing the effect of finite diffusion time and noise scale. In particular, our formula shows we can recover the sampling distribution from the generative model only if the noise scale is small enough compared to certain function of diffusion time. 
    \item We establish that  aforementioned  Kullback-Leibler divergence starts to decrease monotonically with addition of new training data when the size of the training set reaches the dimension of the data points as opposed to an exponential scale indicated by the curse of dimensionality. 
    \item For a realistic diffusion model on Gaussian mixture  training set, we quantify the fact that larger diffusion step leads to better production quality. In addition, we show that the theorem that we proved before  gives us theoretical explanation of this fact. 
\end{enumerate}

\subsection{Related works}

The main theoretical setup of our work is that of higher dimensional statistics, i.e, when the dimension and number of train data size both scale large simultaneously staying proportional to each other.  In the context of  linear regression \cite{krogh1992generalization, dicker2016minimax,dobriban2018prediction, nakkiran2019more, advani2020high, hastie2022surprises}, kernel regression \cite{sollich1998learning, sollich2002learning, bordelon2020spectrum, canatar2021spectral, spigler2020asymptotic, simon2023eigenlearning, loureiro2021learning}, and random feature models \cite{hastie2022surprises, louart2018random, mei2022generalization, adlam2020neural, d2020double, d2020triple, loureiro2021learning, bahri2021explaining, zavatone2023learning,dhifallah2020precise, hu2022universality, maloney2022solvable, bach2024high} method of deterministic equivalence has been used extensively for discussions of higher dimensional statistics. 

Traditionally diffusion models are  trained with datasets whose size is much smaller compared to the exponential of the data dimension. 
For example a comonly used dataset for training diffusion models is laion-high-resolution that contains around $10^8$ images of dimension $1024\times 1024$.
Motivated by these facts, in this paper we study a specific linear denoiser based generative model that captures a single diffusion step of the realistic diffusion model using the method of deterministic equivalence. For stochastic differential equation based models, there are notable works discussing  bound on the distance between sampling and generated distribution under the assumption of a given score estimation error \cite{lee2023convergencescorebasedgenerativemodeling, chen2023samplingeasylearningscore, chen2023improvedanalysisscorebasedgenerative, benton2024nearlydlinearconvergencebounds, shah2023learningmixturesgaussiansusing}. On the other hand,  in our work, we focus on the error in denoising for a single step taking into account finite sample size.

\section{A  generative model based on a linear denoiser}\label{secS2}

In this section, we define and study a linear  denoiser based generative model which is analogous to a one step diffusion model. Before explaining the model, we note certain basic facts about the diffusion process based on a finite number of samples. Given $n$ samples  $\rho(x)$ can be approximated by the Dirac delta distribution $\hat{\rho}(0,x)\equiv \frac{1}{n}\sum_{k=1}^n \delta (x-x_k), x \in \mathbb{R}^d$.
The time evolution of the probability distribution under Ornstein-Uhlenbeck diffusion process is given by (see Appendix \ref{appA} for more details)
\begin{equation}\label{app_rho}
\begin{aligned}
    &  \hat{\rho}(t,x)= \frac{1}{n}\sum_{k=1}^n  \ \mathcal{N}(x|x_k e^{-t},1-e^{-2t})\\
\end{aligned}   
\end{equation}
This motivates us to sample $Y_k, k=1,2,..,n$ from the underlying distribution $\rho(x)$ and add noise $ Z_k \sim  \mathcal{N}(0,\mI_d)$ to it to obtain noisy samples
\begin{equation}\label{add_noise}
    X_k=e^{-T}Y_k+\sqrt{ \Delta_T} Z_k, \Delta_T=\lambda\delta_T=\lambda (1-e^{-2T})
\end{equation}
Here $T$ is the diffusion time cut-off 
and $\lambda$ is a free hyperparameter that controls the amount of noise added \footnote{This corresponds to scaling the noise term in (\ref{OUdynamics}) by a factor of $\sqrt{\lambda}$.}. 

The denoiser based 
model,  trained on the data above, as input takes a noisy sample $X$ and generates a clean sample $Y$. In this paper, we consider a linear model $Y=\hat{\theta}_0+\hat{\theta}_1 X$ as prototype denoiser for analytical tractability. The parameters $\hat{\theta}_0, \hat{\theta}_1$ are solution to the linear regression problem of predicting $\{Y_k\}$ given $\{X_k\}$ and given by\footnote{A natural generalization of this is to feature kernel regression instead of linear regression.  }
\begin{equation}\label{gendLR}
    \begin{aligned}
       & \hat{\theta}_1^T=(x^Tx)^{-1}x^Ty, \quad \hat{\theta}_0=\hat{Y}-\theta_1 \hat{X}
       & \hat{X}=\frac{1}{n}\sum_{k=1}^n X_k, \quad \hat{Y}=\frac{1}{n}\sum_{k=1}^n Y_k\\
    \end{aligned}
\end{equation}
Here $x,y$ are $n\times d$ dimensional matrices 
whose $k$-th row is $(X_k-\hat{X})^T,(Y_k-\hat{Y})^T$ respectively.

To generate samples from the trained diffusion model we first draw $X$ from  $ \mathcal{N}(\mu_X,\sigma_X^2\mI_d)$ where
with
\begin{equation}
    \mu_X=e^{-T}\hat{Y}, \quad \sigma_X^2=e^{-2T}\frac{1}{nd}\sum_{k=1}^n |Y_k-\hat{Y}|^2+\Delta_T
\end{equation}
and then use the diffusion model to predict corresponding $Y=\hat{\theta}_0+\hat{\theta}_1 X$.  The generated probability distribution for a given  set $\{(X_k,Y_k), k=1,2,..,n\}$ is
\begin{equation}\label{rhoG}
    \begin{aligned}
       & \rho_G(Y|\{(X_k,Y_k)\})
        =\mathcal{N}(Y|\hat{\theta}_0+\hat{\theta}_1\mu_X,\sigma_X^2\hat{\theta}_1^T\hat{\theta}_1)\\
    \end{aligned}
\end{equation}

\subsection*{Effect of finite diffusion time}

In this subsection, we study the effect of finite diffusion time $T$ and noise scale $\lambda$ on generalization error for the linear diffusion model defined above. We restrict our discussion to sampling from isotropic Gaussian distribution  $\rho=\mathcal{N}(\mu, \sigma^2\mI_d)$. The distance between  the underlying distribution $\rho$ and the generated $\rho_G$ distribution from the diffusion model as given in (\ref{rhoG}) can be  measured in terms of Kullback–Leibler divergence. Further it can be decomposed as $\text{KL}(\rho||\rho_G)=\text{KL}_{\text{mean}}+\text{KL}_{\text{var}}\geq \text{KL}_{\text{var}} $. Where the contributions $\text{KL}_{\text{mean}},\text{KL}_{\text{var}}$ are related to the difference between generated and the underlying distribution in mean and variance
\begin{equation}\label{KL-parts}
    \begin{aligned}
    &\text{KL}_{\text{mean}}(\rho_G|\rho)
     =\frac{1}{2\sigma^2} (\mu-\hat{\mu}_G)^T(\mu-\hat{\mu}_G), \quad   
      \text{KL}_{\text{var}}(\rho_G|\rho)=\frac{1}{4}\Tr\bigg(\bigg(\frac{\hat{\Sigma}_G}{\sigma^2}-I\bigg)^2\bigg)
    \end{aligned}
\end{equation}
Here $\hat{\mu}_G=\hat{\theta}_0+\hat{\theta}_1\mu_X, \hat{\Sigma}_G=\sigma_X^2\hat{\theta}_1^T\hat{\theta}_1$.
The inequality follows because $\text{KL}_{\text{mean}}$ is a positive semi-definite quantity.  We numerically show in figure \ref{fig1} that there exists a regime of small $\lambda$ where $\text{KL}_{\text{mean}}\ll \text{KL}_{\text{var}}$ making the inequality above an approximate equality.

\begin{figure*}[t]
    \centering
    \makebox[0.8\textwidth][c]{%
        \begin{tabular}{cc}
            \includegraphics[height=1.2in]{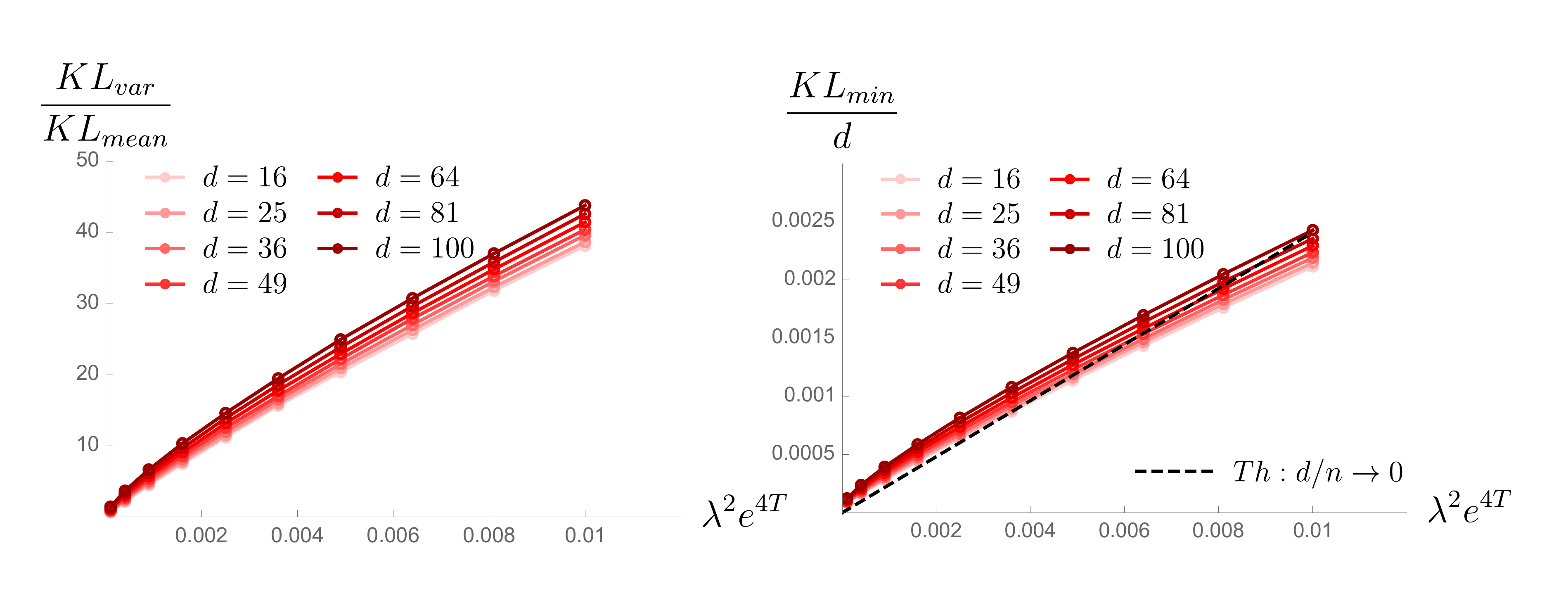} &
            \includegraphics[height=1.2in]{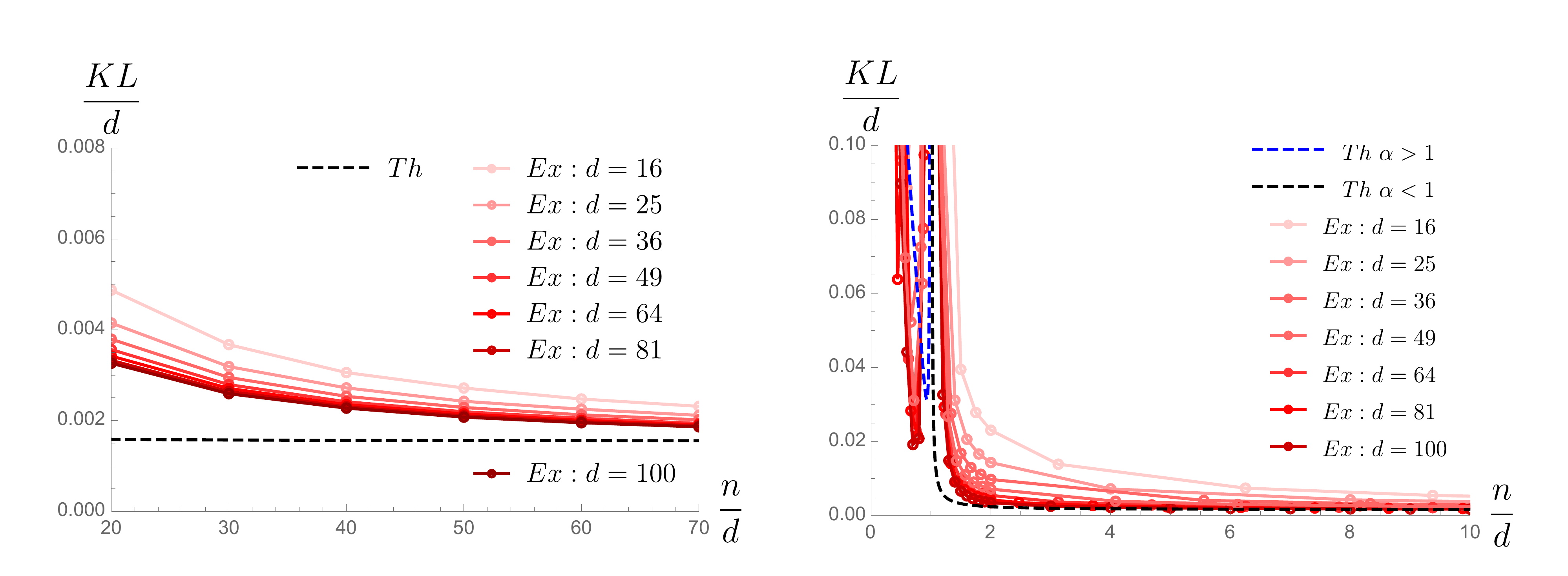} \\
            \multicolumn{1}{c}{\footnotesize (a)} &
            \multicolumn{1}{c}{\footnotesize (b)} \\
        \end{tabular}
    }
    \caption{In figure (a), we plot various contributions to KL divergence between the generated data from the linear denoiser based generative model and sampling distribution taken to be an isotropic Gaussian  of mean $\mu=10$ and diagonal standard deviation $ \sigma=1$ of dimension $d$. We have fixed the diffusion time cut-off $T=2$ and varied the noise scale $\lambda=\hat{\lambda}e^{-4}$. The train dataset size  $n=10^4 \gg d$. From the plot on left we see there exists a regime of parameters when  $\text{KL}_{\text{mean}}\ll \text{KL}_{\text{var}}$. This justifies our assumption of ignoring $\text{KL}_{\text{mean}}$ in analytic calculation presented in appendix \ref{appkl}.  The plot on the right compares the numerical results against the theoretical result and shows that the minimum KL divergence attainable in $d/n \to 0$ limit scales quadratically with the noise parameter $\hat{\lambda}=\lambda e^{2T}/\sigma^2$ for small values of the later. In figure (b), we have fixed the diffusion time to be $T=2$ with noise scale $\lambda=0.8 e^{-4}$.  On the left, we plot KL divergence between the generated and sampling distribution after truncation to the quadratic order in $\lambda$, as given in (\ref{KL2}), in the regime of small $d/n$ comparing experimental data  (in red) and theoretical result for the lower bound as given in (\ref{th1}) (in black). The numerical results on the right plot shows that KL divergence between the generated and underlying distribution scales as $d$ times solely a function of $\alpha=d/n$ without additional $n,d$ dependence as we take $n,d$ large keeping their ratio fixed. This fact is analytically established in (\ref{th1}),(\ref{th2}) and the analytical expression is plotted in black for $\alpha<1$ and blue $\alpha>1$.}\label{fig1}
\end{figure*}

\begin{theorem}
\label{thm:bigtheorem}
When the linear diffusion model described above is trained on $n$ samples from isotropic Gaussian distribution $\rho=\mathcal{N}(\mu, \sigma^2\mI_d)$ in the limit of $n\to \infty$ holding $\alpha=d/n$ fixed, following lower bound on the KL divergence between generated and sampling distribution holds
$\text{KL}(\rho||\rho_G)\geq \text{KL}_{\text{var}} $ 
If further we restrict ourselves to small noise scale  $\lambda=\hat{\lambda} \sigma^2e^{-2T}, \hat{\lambda}\ll 1 $, then an explicit expression for the statistical expectation value of $\text{KL}_{\text{var}}$ can be obtained order by order in $\hat{\lambda}$ based on the theory of deterministic equivalence. More specifically for $\alpha<1$ we have
\begin{equation}\label{th1}
    \begin{aligned}
       \langle  \text{KL}_{\text{var}} \rangle=&\frac{d\alpha  \hat{\lambda } e^{-4 T} \left(e^{2 T}-1\right)}{2 (1-\alpha)}+\frac{d\hat{\lambda }^2 e^{-8 T} \left(e^{2 T}-1\right)^2}{4 (1-\alpha)^3} \left(\alpha ^2+(1-\alpha)^3 e^{4 T}+4 \alpha  (1-\alpha)^2 e^{2 T}\right)+\mathcal{O}(\hat{\lambda}^3)
    \end{aligned}
\end{equation}
and for $\alpha>1$ 
\begin{equation}\label{th2}
    \begin{aligned}
        \langle  \text{KL}_{\text{var}} \rangle=& d\frac{\alpha -1}{4 \alpha }+ \frac{d\hat{\lambda } e^{-4 T} \left(e^{2 T}-1\right)}{2 (\alpha -1)}+\frac{d\hat{\lambda }^2 e^{-8 T} \left(e^{2 T}-1\right)^2}{4 (\alpha -1)^3 \alpha } (\alpha ^3+(\alpha -1)^3 e^{4 T} +4 \alpha  (\alpha -1)^2 e^{2 T})+\mathcal{O}(\hat{\lambda}^3)
    \end{aligned}
\end{equation}
\end{theorem}
\begin{proof}
    See Appendix \ref{appkl}.
\end{proof}
\begin{lemma}\label{lemma2}
 For $n>d$,  $\text{KL}_{\text{var}}$ is a monotonically decreasing function of $n/d$.
\end{lemma}
\begin{proof}
    Derivative of RHS in (\ref{th1}) with respect to $\alpha$ is positive. 
\end{proof}
\begin{lemma}\label{lemma3}
   In $n/d\to \infty$ limit, $\text{KL}_{\text{var}}/d$ scales as $\lambda^2e^{4T}(1-e^{-2T})^2$. Hence we conclude we can recover the underlying sampling distribution  in this limit only if $\lambda  e^{2T}(1-e^{-2T})\ll1$.  
\end{lemma}
\begin{proof}
    Consider $\alpha\to 0$ limit of RHS in (\ref{th1}).
\end{proof}

 Thse are in agreement with the plot in figure \ref{fig1}.

\section{Non-linear diffusion model}

\begin{figure*}[h]
    \centering
    \makebox[0.8\textwidth][c]{%
        \begin{tabular}{cc}
            \includegraphics[height=1.2in]{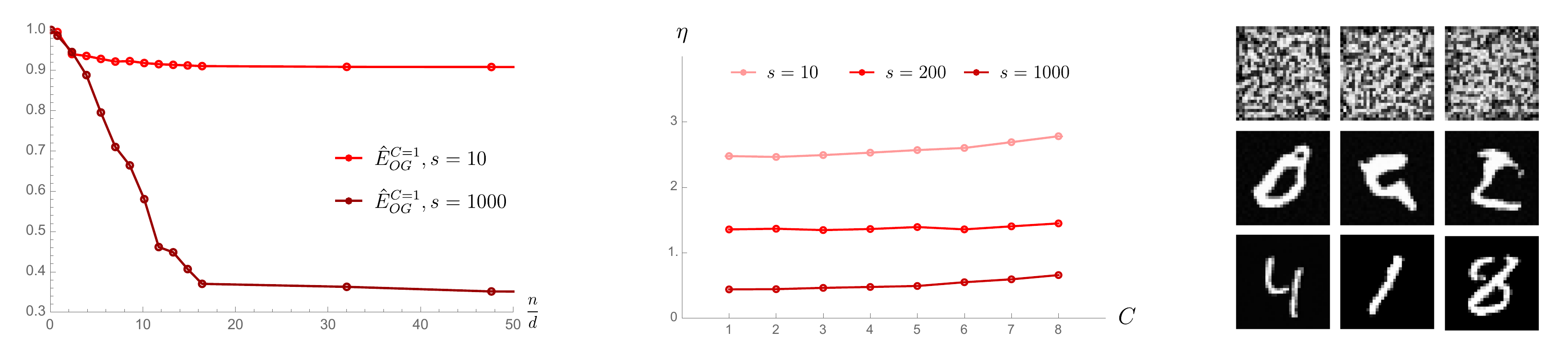}  \\
            \multicolumn{1}{c}{\footnotesize }  \\
        \end{tabular}
    }
    \caption{ 
The plot is based on
    PyTorch-based implementation of the algorithm in \cite{ho2020denoising}. The denoiser has the structure of U-Net \citep{ronneberger2015unetconvolutionalnetworksbiomedical} with additional residual connections consisting of positional encoding of the image and attention layers \citep{dosovitskiy2021an, tu2022maxvit, peebles2023scalable}.
    The train dataset is equal weight Gaussian mixture model (GMM) with $C$ components $i=0,1,2,..,C-1$  of dimension $d=64$. The $i$-th component is an isotropic Gaussian of mean $\mu_i=\mu_0+(i-(C-1)/2)\sigma_0, \mu_0=0.5$ and standard deviation $\sigma_i=0.1$.
  The number of samples in the original dataset is $N=10^4$. Plot on the left shows how increasing diffusion steps reduces  $\hat{E}_{OG}$ - it is the scaled value of the distance between generated and original distribution $E_{OG}$ by its value at $n=10$ for $C=1$. For the plot in the middle the
  training is done for $10$ epoch with batch size $128$. 
 The plot in the middle shows a linear dependence of error $\eta =-\log(1-E_{OG}(d/n\to 0)/E_{OG}(d/n \to \infty)) $ on sample complexity $C$. On the right, we have generated images from the model with same hyper-parameter configurations on MNIST dataset. We can clearly see that image quality improves as diffusion steps increase.  }\label{fig2}
\end{figure*}

From figure \ref{fig2}, it is clear that  production quality of a realistic diffusion model improves with higher diffusion steps $s$. This fact can be explained based on our theoretical analysis as follows. In the diffusion model we start from a clear image $x_0$ and then obtain noisy images for steps $t=1,2,...,s$ from (equation (2) of \cite{ho2020denoising})
\begin{equation}
    \begin{aligned}
        x_{t}=\sqrt{1-\beta_t}x_{t-1}+\sqrt{\lambda} \sqrt{\beta_t}Z, \quad  Z\sim \mathcal{N}(0,\mI_d)
    \end{aligned}
\end{equation}
Our observation is that it becomes identical to (\ref{add_noise}) if we choose the noise schedule as follows
\begin{equation}
    \beta_t=1-e^{-2\beta t/s}=(2t/s)\beta+\mathcal{O}(\beta^2), \beta \ll 1
\end{equation}
with the map $x_t \to X, x_{t-1} \to Y, \beta t/s \to T$.
 If instead of using a non-linear neural network we use the linear denoiser to predict a Gaussian approximation of $x_{t-1}$ from $x_t$ we can use lemma \ref{lemma3} and find that $\text{KL}_{\text{var}}(t)/d$ scales as $\lambda^2 \beta_t^2 = 4\lambda^2\beta^2 (t/s)^2 $. This shows that performance of the final denoising step is improved by a factor of $s^2$ compared to single step diffusion. In fact, performance of each step improves except the first one between $x_s, x_{s-1}$ which remain the same. This suggests as we increase $s$ overall production quality of the model will improve in agreement with the findings in figure \ref{fig2}.\footnote{For this purpose, the distance between original and generated distribution is calculated using
 \begin{equation}\label{errors}
    \begin{aligned}
     & \text{E}_{OG}= \sum_{i=1}^{d}\int dx \  (\rho_{O,i}(x)-\rho_{G,i}(x))^2, 
     &  \rho_{O/G,i}(x) \equiv \frac{1}{|S_{O/G}|}\sum_{k=1}^{|S_{O/G}|} \mathcal{N}(x|x^i_k, \epsilon_{O,i}^2), \quad \epsilon_{O,i}^2 = \frac{\hat{\Sigma}_{O,i}}{|S_{O/G}|^2}
    \end{aligned}
\end{equation}
Here $ \hat{\Sigma}_{O,i}$ is the empirical variance obtained from the original dataset $S_O$.
 }  This argument also predicts that we need to have $\lambda \beta \ll 1$ for good quality generated images.

The choice of different noise schedules $ \beta_t$ \cite{nichol2021improveddenoisingdiffusionprobabilistic,jabri2023scalableadaptivecomputationiterative, chen2023importancenoiseschedulingdiffusion} induces distinct linear models for each single step of the diffusion process. In the future, these models can be analyzed using techniques analogous to those employed in the present study, thereby enabling a systematic comparison of the performance associated with various noise schedules.

Moreover, the analytical framework developed here may be extended to encompass the analysis of (a stack of) wide neural networks, either in the kernel approximation regime \citep{jacot2020neuraltangentkernelconvergence, NEURIPS2018_5a4be1fa, NEURIPS2019_0d1a9651,bordelon2021spectrumdependentlearningcurves, Canatar_2021, Roberts_2022, atanasov2023the, demirtas2023neuralnetworkfieldtheories} or in mean field regime \citep{Mei_2018, pmlr-v139-yang21c, bordelon2022selfconsistentdynamicalfieldtheory, bordelon2023dynamicsfinitewidthkernel}  in place of the linear diffusion model considered in this work.

\appendix

\section{Foundations of diffusion-driven generative models}\label{appA}
In this Appendix we review the connection between stochastic interpolant and stochastic differential equation based generative models \cite{albergo2023stochastic, song2021scorebased, karras2022elucidatingdesignspacediffusionbased, huang2021variationalperspectivediffusionbasedgenerative, zhang2021diffusionnormalizingflow}. 
Given two probability density functions $\rho_0, \rho_1$, one can construct a  stochastic interpolant between $\rho_0$ and $\rho_1$ as follows
\begin{equation}
    x(t) = X(t,x_0,x_1) + \lambda_0(t) z,  \qquad t\in [0, 1]
\end{equation}
where the function $X, \lambda_0$ satisfies
\begin{equation}
    \begin{aligned}
        & X(0,x_0,x_1) = x_0, \quad X(1,x_0,x_1) = x_1,
        &||\partial_t X(t,x_0,x_1)||\le C||x_0-x_1||\\
        & \lambda_0(0)=0, \quad  \lambda_0(1) = 0,\quad  \lambda_0(t) \geq 0
    \end{aligned}
\end{equation}
for some positive constant $C$.
Here  $x_0,x_1,z$ are drawn independently from a probability measure $\rho_0$, $\rho_1$ and standard normal distribution $\mathcal{N}(0,\mI)$.
The probability distribution $\rho(t,x)$ of the process $x(t)$ satisfies the transport equation\footnote{Here we are using the notation $\nabla=\nabla_x$.}  
\begin{equation}
    \label{transport}
    \partial_t \rho + \nabla \cdot \left(b_\ODE\rho\right) = 0, \quad \rho(0,x)=\rho_0(x), \quad \rho(1,x)=\rho_1(x), 
\end{equation}
where we defined the velocity\footnote{The expectation is taken independently over $x_0\sim \rho_0, x_1\sim \rho_1$ and $z\sim \mathcal{N}(0,\mI).$ Here $\mathcal{N}(0,\mI)$ is normalized Gaussian distribution of appropriate dimension with vanishing mean and variance.}
\begin{equation}
    b_\ODE(t,x) = \EE [ \dot x(t)| x(t) = x] = \EE [ \partial_t X(t,x_0,x_1) + \dot \lambda_0(t) z| x(t) = x].
\end{equation}
One can estimate the velocity field by minimizing 
\begin{equation}
    \label{Lb}
   L_b[\hat{b}] =\int_0^1   \EE \left( \frac{1}{2} ||\hat b(t,x(t))||^2 - \left(\partial_t X(t,x_0,x_1) + \dot \lambda_0(t) z \right) \cdot \hat b(t,x(t)) \right) dt
\end{equation}

It's useful to introduce the score function $s(t,x)$ for the probability distribution for making the connection to the stochastic differential equation
   \begin{equation}
        s(t,x) = \nabla \log\rho(t,x) = - \lambda_0^{-1}(t) \EE( z |x(t)=x) 
    \end{equation}
It can be estimated by minimizing
\begin{equation}
    \label{Ls}
   L_s[\hat{s}] = \int_0^1 \EE\left( \frac{1}{2} ||\hat s(t,x(t))||^2 +\lambda_0^{-1}(t) z\cdot \hat s(t,x(t)) \right) dt
\end{equation}
The score function also can be obtained by minimizing the following alternative objective function known as the Fisher divergence
\begin{equation}
    \label{Fisher}
    \begin{aligned}
       L_F[\hat{s}] &= \tfrac12 \int_0^1 \EE\left( ||\hat s(t,x(t))-\nabla\log \rho(t,x)||^2 \right) dt \\
        &=\int_0^1 \EE\left( \frac{1}{2}||\hat s(t,x(t))||^2 +\nabla\cdot \hat s(t,x(t)) +\frac{1}{2} ||\nabla\log \rho(t,x))||^2\right) dt
    \end{aligned}
\end{equation}
To obtain the second line we have ignored the boundary term. Note that for the purpose of minimization the last term is a constant and hence it plays no role hence Fisher divergence can be minimized from a set of samples drawn from $\rho$ easily even if the explicit form of $\rho$ is not known \citep{JMLR:v6:hyvarinen05a}. However, the estimation of $\nabla\cdot \hat s(t,x(t))$ is computationally expensive and in practice one uses denoising score matching for estimating the score function \citep{Vincent}. 

It is easy to put eq. (\ref{transport}) into Fokker-Planck-Kolmogorov form
\begin{equation}
\begin{aligned}
    & \partial_t \rho + \nabla \cdot \left(b_{F}\rho\right) = +\lambda(t)  \Delta \rho, \qquad b_F(t,x) = b_\ODE(t,x) + \lambda(t) s(t,x)\\
     & \partial_t \rho + \nabla \cdot \left(b_{B}\rho\right) = -\lambda(t)  \Delta \rho, \qquad b_B(t,x) = b_\ODE(t,x) - \lambda(t) s(t,x)
\end{aligned}
\end{equation}
For an arbitrary function $\lambda(t)\geq0$. From this, we can read off the Itô SDE as follows\footnote{Here $ W_t$ represents a standard Wiener process, i.e., $ W_t - t W_1=N_t$ is a zero-mean Gaussian stochastic process that satisfies $\EE [N_t N^\T_t] = t(1-t)\mI$. }
\begin{equation}\label{ItoSDE}
\begin{aligned}
    & dX^F_t = b_F(t,X^F_t)dt  + \sqrt{2\lambda(t)} \,  dW_t\\
    & dX^B_t = b_B(t,X^B_t)dt  - \sqrt{2\lambda(t)} \,    dW_{1-t}
\end{aligned}
\end{equation}
The first equation is solved forward in time from the initial data~$X^F_{t=0}\sim\rho_0$ and the second one is solved backward in time from the final data~$X^B_{t=1}\sim\rho_1$.
One can recover the probability distribution $\rho$ from the SDE using Feynman–Kac formulae\footnote{A class of exactly solvable models are given by (Ornstein-Uhlenbeck dynamics discussed in the main text is a special case of this equation)
\begin{equation}
    dX^F_t=X^F_t\frac{d}{dt}(\log \eta(t)) dt+ \sqrt{\eta(t)^2 \frac{d}{dt} \bigg(\frac{\sigma(t)^2}{\eta(t)^2}\bigg)}dW_t, \quad X^F_t \sim \mathcal{N}(\eta(t)X^F_0,\sigma(t)^2)
\end{equation}
Where $\eta, \sigma$ are two positive functions satisfying $\eta(0)=1, \sigma(0)=0$.
}
\begin{equation}
\begin{aligned}
     \rho(t,x) &= \EE \left(e^{ \int_t^0\nabla \cdot  b_F(t, Y^B_t) dt}  \rho_0(Y_{t=0}^B)| Y_{t}^B=x\right)\\
     &= \EE \left(e^{ \int_t^1\nabla \cdot  b_B(t, Y^F_t) dt}  \rho_1(Y_{t=1}^F)| Y_{t}^F=x\right)
\end{aligned}
\end{equation}

In the domain of image generation, we don't know the exact functional form of the sampling distribution $\rho(x)$. However we have access to a finite number of samples from it and the goal  is to generate more data points from the unknown probability density $\rho(x)$. Traditional likelihood maximization techniques would assume a trial density function $\rho_\theta$ and try to adjust $\theta$ so that likelihood for obtaining known samples is maximized. In this process determination of the normalization of $\rho_\theta$ is computationally expensive as it requires multi-dimensional integration (typically it is required for each step of the optimization procedure for $\theta$). Diffusion based generative models are an alternative \citep{sohldickstein2015deepunsupervisedlearningusing, song2019generative, ho2020denoising}. In this section, we review basic notions of these stochastic differential equation based models. In particular, we examine an exactly solvable stochastic differential equation (SDE).
The Itô SDE under consideration is known as the Ornstein-Uhlenbeck Langevin dynamics and is expressed by:
\begin{equation}\label{OUdynamics}
    dX^F_t=-X^F_t dt+ \sqrt{2}dW_t, \quad X^F_t \sim \rho(t).
\end{equation}
The score function associated with the stochastic process will be denoted  as 
\begin{equation}\label{appS}
     s(t,x) = \nabla_x \log\rho(t,x) =\frac{1}{\rho(t,x)} \nabla_x \rho(t,x)
\end{equation}
The probability density $\rho$ satisfies the transport equation  (see (\ref{transport}))
\begin{equation}\label{appT}
    \begin{aligned}
        \partial_t \rho(t,x) &= \nabla \cdot \left((x+s(t,x)) \rho(t,x)\right) \\
        &= \nabla^2 \rho(t,x)+x.\nabla \rho(t,x)+d \rho(t,x).
    \end{aligned}
\end{equation}
The dimension of the data is defined to be given by $d=\text{dim}(x)$. The time evolution of the probability distribution is exactly solvable and given by
\begin{equation}\label{exactrho}
    \rho(t,X^F_t)=\int dX^F_0 \ \rho(0,X^F_0) \ \mathcal{N}(X^F_t|X^F_0e^{-t},1-e^{-2t}).
\end{equation}
Suppose we know the probability density $\rho(0,x)$ exactly. One way to sample from it would be to use the knowledge of the exact score function $s(t,x)$ in the reverse  diffusion process  (see (\ref{ItoSDE})), i.e,
\begin{equation}\label{revOU}
    \begin{aligned}
        & dX^B_t = (-X_t^B-2s(t,X_t^B))dt  - \sqrt{2} \,    dW_{1-t}
    \end{aligned}
\end{equation}
starting from a late time distribution $\rho(T,x)$ (it is assumed that we know how to sample from $\rho(T,x)$).

\section{Principle of deterministic equivalence}

In this appendix we review the theory of large random matrices leading to the principle of deterministic equivalence. A $d \times d$ Hermitian random matrix $A$ with measure $d\mu_A$ is called an invariant random matrix if the measure satisfies 
\begin{equation}
    d\mu_A(A)=d\mu_A(U^\dagger A U)
\end{equation}
for any unitary matrix $U$. In the limit of $d\to \infty$, the theory is conveniently described in terms of the single eigenvalue density $\rho_A$ (normalized to unity) that can be obtained from the resolvent or the Stieltjes transform
\begin{equation}\label{S-transform}
    G_A(z)=\langle \frac{1}{d}\Tr \bigg( \frac{1}{z-A}\bigg) \rangle= \int \frac{\rho_A(\lambda)d\lambda}{z-\lambda} \implies \rho_A(\lambda)=-\frac{1}{\pi}\lim_{\epsilon \to 0+}\Im(G_A(\lambda+i\epsilon))
\end{equation}
The moment generating function is given by
\begin{equation}
    M_A(z)=\frac{1}{z}G_A\bigg(\frac{1}{z}\bigg)-1=\langle\frac{1}{d}\Tr \sum_{i=1}^\infty A^i z^i\rangle
\end{equation}
R, S transformation of the eigenvalue density are defined by
\begin{equation}
\begin{aligned}
    & R_A(z)=G_A^{inv}(z)-\frac{1}{z}, \quad G_A^{inv}(G_A(z))=z\\
    &  S_A(z)=\frac{z+1}{z}M_A^{inv}(z), \quad M_A^{inv}(M_A(z))=z
\end{aligned} 
\end{equation}
R, S transformations are useful when we study the matrix model to the leading order in large $d$ limit as we explain next. Two invariant random matrices $A,B$ are called free to the leading order in large $d$ limit if they are independent.  Free sum and free product of $A,B$ are defined as follows
\begin{equation}
    \begin{aligned}
       & A  \boxplus B = U^\dagger A U+V^\dagger B V\\
       & A \star B=\sqrt{A} B \sqrt{A}
    \end{aligned}
\end{equation}
Here $U,V$ are are sampled independently from uniform measure on the unitary group, i.e., Haar random unitary.
It can be shown that for two invariant, independent random matrices $A,B$ the moment generating function of $A \boxplus B$ and $A+B$ coincides, similarly moment generating function of $A\star B$ and $AB$ coincides (to the leading order in large $d$, i.e., when they are free). Furthermore following identity holds to the leading order in large $d$ limit for two free matrices $A,B$
\begin{equation}
    R_{A\boxplus B}(z)=R_{A}(z)+R_{B}(z), \quad S_{A\star B}(z)=S_{A}(z) S_{B}(z)
\end{equation}

Now we turn to application of these ideas. Consider $d\times d$ matrix $\Sh=x^T x/n$ where each row of $x$ (there are $n$ rows) is drawn from $\mathcal{N}(0,\Sigma)$. Then it can be written as a free product of $\Sigma$ and a white Wishart matrix $W$ (corresponds to $x^T x/n$ where each row of $x$ is drawn from $\mathcal{N}(0,\mI_d)$): $\Sh=\Sigma\star W$. From the definition of S transformation it follows that
\begin{equation}
    \begin{aligned}
        M_{\Sh}(z)=\frac{1}{\frac{S_{\Sh}(M_{\Sh}(z))}{z}-1}=\frac{1}{\frac{S_{\Sigma}(M_{\Sh}(z)) S_{W}(M_{\Sh}(z))}{z}-1}=M_{\Sigma}\bigg(\frac{z}{S_{W}(M_{\Sh}(z))}\bigg)
    \end{aligned}
\end{equation}
To obtain the final equality we used self-consistency of the equation itself. To recast this equation in a compact way we define $\df^1_A(z)=-M_A(-1/z)= \langle \Tr  \Sh(\Sh+\Rh)^{-1} \rangle/d$. In terms of this new quantity we have
\begin{equation}
    \begin{aligned}
        \df^1_{\Sh}(\Rh)  =   \df^1_{\Sigma}(R) , \quad \Rh=R(1-\alpha \ \df^1_{\Sigma} (R))
    \end{aligned}
\end{equation}
To obtain this equation we used knowledge of S transformation of white  Wishart matrices $S_W(z)=1/(1+\alpha z), \alpha=d/n$. This equation is valid only leading order in large $d,n$ limit with fixed $\alpha$. It is known as the principle of deterministic equivalence.  See \cite{atanasov2024scalingrenormalizationhighdimensionalregression} and references therein for a recent discussion of it.

\section{Generalization error from deterministic equivalence}\label{appkl}

In this appendix, we provide proof of the main theorem in the paper (\ref{th1}), (\ref{th2}). Consider the scenario when the underlying sampling distribution is an isotropic Gaussian $\rho=\mathcal{N}(\mu,\sigma^2I_d)$. The linear diffusion model $Y=\theta_0+\theta_1 X$  is trained to solve the following linear regression problem 
\begin{equation}\label{model_assumption}
\begin{aligned}
   & Y_k=e^T X_k+ Z_k, \quad X_k \sim \mathcal{N}(\mu_X, \Sigma=\sigma_X^2I_d), \quad Z_k \sim \mathcal{N}(0,\Delta_TI_d), \quad k=1,\dots, n\\
   &    \mu_X=e^{-T}\mu, \quad \sigma_X^2=e^{-2T}\sigma^2 +\Delta_T,\quad \Delta_T=\lambda(1-e^{-2T})\\
\end{aligned}
\end{equation}
Here $X_k,Z_k$ are taken independent of each other.\footnote{Comparing first and second moment of $Y_k=a(T)X_k+c(T)+Z_k$, for $X_k$ given as in (\ref{add_noise}), with the expected answer, we can see that in the domain considered here, i.e, $\lambda=\hat{\lambda} \sigma^2e^{-2T}, \hat{\lambda}\ll1  $, the solution to the leading order in $\hat{\lambda}$ is indeed given by (\ref{model_assumption}).}
The optimal value of the weights $\hat{\theta}_0,\hat{\theta}_1 $ that minimizes the standard square loss are given by 
\begin{equation}
    \begin{aligned}
         & \hat{\theta}_1^T=(x^Tx+n \hat{R})^{-1}x^Ty, \quad \hat{\theta}_0=\hat{Y}-\hat{\theta}_1 \hat{X} \quad  \hat{X}=\frac{1}{n}\sum_{k=1}^n X_k, \quad \hat{Y}=\frac{1}{n}\sum_{k=1}^n Y_k\\
    \end{aligned}
\end{equation}
Here $x,y$ are $n\times d$ dimensional matrices  whose $k$-th row is $(X_k-\hat{X})^T,(Y_k-\hat{Y})^T$ respectively (e.g. $x_{iA}=(X_i-\hat{X})_A$ etc.) and  $\hat{R}$ is a scalar ridge parameter. Once trained the diffusion model generates data from $\rho_G=\mathcal{N}(\hat{\mu}_G,\hat{\Sigma}_G) $, $\hat{\mu}_G=\hat{\theta}_0+\hat{\theta}_1\mu_X, \hat{\Sigma}_G=\sigma_X^2\hat{\theta}_1^T\hat{\theta}_1$.

KL divergence between two PDF $\rho_1=\mathcal{N}(\mu_1,\Sigma_1), \rho_2=\mathcal{N}(\mu_2,\Sigma_2)$ is given by
\begin{equation}
    \begin{aligned}
     \text{KL}(\rho_1|\rho_2)=& \int \rho_1(x)\log \frac{\rho_1(x)}{\rho_2(x)} dx\\
     =&\frac{\Tr(\Sigma_2^{-1}\Sigma_1)-\Tr(I)}{2}- \frac{1}{2}\log |\Sigma_2^{-1}\Sigma_1|+\frac{1}{2} (\mu_1-\mu_2)^T\Sigma_2^{-1}(\mu_1-\mu_2)
    \end{aligned}
\end{equation}
We choose $\mu_2=\mu, \Sigma_2=\sigma^2 I_d$ to correspond to the underlying distribution and $ \mu_1=\hat{\mu}_G,\Sigma_1=\hat{\Sigma}_G$
corresponds to the generated distribution. This simplifies the formula above to
\begin{equation}\label{KL1}
    \begin{aligned}
     \text{KL}(\rho_G|\rho)
     =&\frac{1}{2}(\Tr\bigg(\frac{\hat{\Sigma}_G}{\sigma^2}\bigg)-\Tr(I))- \frac{1}{2}\Tr \log \bigg(\frac{\hat{\Sigma}_G}{\sigma^{2}}\bigg)+\frac{1}{2\sigma^2} (\mu-\hat{\mu}_G)^T(\mu-\hat{\mu}_G)\\
     \geq&\frac{1}{2}(\Tr\bigg(\frac{\hat{\Sigma}_G}{\sigma^2}\bigg)-\Tr(I))- \frac{1}{2}\Tr \log \bigg(\frac{\hat{\Sigma}_G}{\sigma^{2}}\bigg)
    \end{aligned}
\end{equation}
To go to the second line we have ignored the positive semi-definite term related to difference in mean between generated and underlying distribution. We proceed to calculate the variance term in KL divergence above. It follows that
\begin{equation}
\begin{aligned}
    \hat{\theta}_1^T&=(x^Tx+n \hat{R})^{-1}x^Ty\\
    &=(x^Tx+n \hat{R})^{-1}x^T(x \bar{\theta}^T_1+ z)\\
    &= e^T(1-\Rh(\Sh+\Rh)^{-1})+(\Sh+\Rh)^{-1}\frac{x^Tz}{n}
\end{aligned}
\end{equation}
We have defined $\bar{\theta}_1=e^T I_d, \Sh=x^T x/n$ for later convenience. Next we calculate
\begin{equation}
\begin{aligned}
     \hat{\theta}^T_1\hat{\theta}_1=& \bar{\theta}_1^T \bar{\theta}_1+\hat{\Sigma}_{\theta_1}\\
     \hat{\Sigma}_{\theta_1}=&(\Sh+\Rh)^{-1}\frac{x^T z}{n}\frac{z^T x}{n} (\Sh+\Rh)^{-1}+e^{2T}\Rh^2(\Sh+\Rh)^{-2}-2e^{2T}\Rh(\Sh+\Rh)^{-1}\\
     &+e^T \left(\frac{z^Tx}{n}(\Sh+\Rh)^{-1}+(\Sh+\Rh)^{-1} \frac{x^Tz}{n}\right)-e^T\Rh \  (\Sh+\Rh)^{-1}\left(\frac{z^Tx}{n}+\frac{x^Tz}{n}\right)(\Sh+\Rh)^{-1}
\end{aligned}
\end{equation}
Plugging this back in the expression of $\hat{\Sigma}_G$ we get
\begin{equation}
\begin{aligned}\label{stdGen}
   \frac{\hat{\Sigma}_G}{\sigma^2}=& (1 +\sigma^{-2}e^{2T}\Delta_T)(I+e^{-2T}\hat{\Sigma}_{\theta_1})= I+\hat{\sigma}_G\\
   \hat{\sigma}_G=&(e^{-2T} +\sigma^{-2}\Delta_T)\hat{\Sigma}_{\theta_1}+e^{2T}\sigma^{-2}\Delta_T I
\end{aligned}
\end{equation}
Appearance of logarithm in the KL divergence makes it difficult to calculate its statistical expectation value. In next sub-section we develop a controlled expansion for this purpose.

\subsection*{Analytic tractability and various approximations}

From (\ref{stdGen}) we see that the generated distribution remains close to the original underlying distribution if both $\Delta_T$ and $\hat{\sigma}_G$ remain small. To this end, we focus on the following limit:

$\lambda=\hat{\lambda} \sigma^2e^{-2T},\hat{R}=\lambda\hat{r} $. In this regime $\hat{\sigma}_G\sim \hat{\lambda}$. Further taking $\hat{\lambda}\ll 1$ makes $\hat{\sigma}_G$ small and  we can approximate 
\begin{equation}
    \begin{aligned}
         \log \bigg(\frac{\hat{\Sigma}_G}{\sigma^{2}}\bigg)=\log (I+\hat{\sigma}_G)=\hat{\sigma}_G-\frac{1}{2}\hat{\sigma}_G^2+\mathcal{O}(\hat{\sigma}_G^3)
    \end{aligned}
\end{equation}
Plugging this back into the expression of KL divergence (\ref{KL1}) we get $ \text{KL}(\rho||\rho_G)=\text{KL}_{\text{mean}}+\text{KL}_{\text{var}}$, where
\begin{equation}\label{KL2}
    \begin{aligned}
    &\text{KL}_{\text{mean}}(\rho_G|\rho)
     =\frac{1}{2\sigma^2} (\mu-\hat{\mu}_G)^T(\mu-\hat{\mu}_G), \quad \text{KL}_{\text{var}}(\rho_G|\rho)=\frac{1}{4}\Tr\bigg(\bigg(\frac{\hat{\Sigma}_G}{\sigma^2}-I\bigg)^2\bigg)
    \end{aligned}
\end{equation}
We focus on the variance term. Plugging back expressions from previous analysis
\begin{equation}\label{KL2}
    \begin{aligned}
     \text{KL}(\rho_G|\rho)_{\text{var}}
     =&\frac{1}{4}\Tr\bigg(\frac{\hat{\Sigma}_G}{\sigma^2}-I\bigg)^2=\frac{1}{4}\Tr (\hat{\sigma}_G^2)\\
     =&\frac{1}{4}\Tr(((e^{-2T} +\sigma^{-2}\Delta_T)\hat{\Sigma}_{\theta_1}+\sigma^{-2}e^{2T}\Delta_T)^2)\\
     =& \frac{1}{4}(e^{-2T} +\sigma^{-2}\Delta_T)^2\Tr \hat{\Sigma}^2_{\theta_1}\\
     &+\frac{1}{2}(e^{-2T} +\sigma^{-2}\Delta_T)\sigma^{-2}e^{2T}\Delta_T\Tr \hat{\Sigma}_{\theta_1}+\frac{d}{4}(\sigma^{-2}e^{2T}\Delta_T)^2
    \end{aligned}
\end{equation}
We are interested in statistical average of the expression above. We consider the following higher dimensional statistics limit: $n \to \infty, d\to \infty$ keeping $\alpha=d/n$ fixed. In this limit, we can take advantage of principle of deterministic equivalence discussed in previous appendix:
\begin{equation}\label{MDE}
\begin{aligned}
\df^1_{\Sh}(\Rh)  =   \df^1_{\Sigma}(R)  , \quad \df^n_{\Sh}(\Rh)=\frac{1}{d} \langle \Tr  \Sh^n(\Sh+\Rh)^{-n} \rangle , \quad \Rh=R(1-\alpha \ \df^1_{\Sigma} (R))
\end{aligned}
\end{equation}
Since $x,z$ are statistically independent and $z$ has zero mean, we get
\begin{equation}
\begin{aligned}
\Tr  \langle   \hat{\Sigma}_{\theta_1} \rangle =&\Tr\langle (\Sh+\Rh)^{-1}\frac{x^T z}{n}\frac{z^T x}{n} (\Sh+\Rh)^{-1}+e^{2T}\Rh^2(\Sh+\Rh)^{-2}-2e^{2T}\Rh(\Sh+\Rh)^{-1} \rangle\\
\end{aligned}
\end{equation}
The first term is simplified after performing statistical average over $z$
\begin{equation}
    \begin{aligned}
        (x^T z z^T x)_{AB} =x^T_{Ai} z_{iC}z_{jC}x_{jB}\to nd \Delta_T \Sh_{AB}
    \end{aligned}
\end{equation}
The factor of $d$ came from sum over $C$ (we are using the convention of repeated index implies sum). The first term becomes $ \alpha \Delta_T$ times
\begin{equation}
    \begin{aligned}
        \Tr\langle (\Sh+\Rh)^{-1}\Sh(\Sh+\Rh)^{-1} \rangle= \Tr\langle\Sh (\Sh+\Rh)^{-2} \rangle = -\partial_{\Rh}\langle \Tr\Sh (\Sh+\Rh)^{-1} \rangle =  -d\partial_{\Rh} \langle \df^1_{\Sigma}(R)  \rangle 
    \end{aligned}
\end{equation}
For the case we are considering,
\begin{equation}\label{formula1}
    \df^1_{\Sigma=\sigma_X^2 I_d}(R)=\df^1_{\Sigma=I_d}(\sigma_X^{-2}R)=\frac{1}{1+\sigma_X^{-2}R}
\end{equation}
Putting these expressions  together the first term becomes 
\begin{equation}
    \begin{aligned}
        \alpha d \Delta_T \frac{\sigma_X^{-2}}{(1+\sigma_X^{-2}R)^2} \partial_{\Rh} R=\alpha d \Delta_T \sigma_X^{-2} \frac{(\df^1_\Sigma(R))^2}{1-\alpha \df^2_{\Sigma}(R)}
    \end{aligned}
\end{equation}
To obtain the second line we have used the following identity
\begin{equation}\label{identity1}
    \df^{n+1}_\Sigma (R)=\left(1+\frac{R}{n}\partial_R \right)\df^{n}_\Sigma (R), \quad  \partial_{\Rh} R=\frac{1}{1-\alpha \df^2_{\Sigma}(R)}
\end{equation}
The third term is $-2e^{2T}$ times 
\begin{equation}
    \begin{aligned}
        \Tr \langle \Rh(\Sh+\Rh)^{-1} \rangle=d(1- \ \df^1_\Sigma(R))
    \end{aligned}
\end{equation}
The second term is $e^{2T}$ times 
\begin{equation}
    \begin{aligned}
        \Tr \langle (\Rh(\Sh+\Rh)^{-1})^2 \rangle&=\Tr \langle 1-2 \Sh (\Sh+\Rh)^{-1}+\Sh^2 (\Sh+\Rh)^{-2}\rangle\\
        &=d-2d \ \df^1_{\Sigma(R)}+d \ \df^2_{\Sh}(\Rh)\\
        &=d-2d \ \df^1_\Sigma(R)+d(1+ \frac{\Rh}{1-\alpha \df^2_\Sigma(R) }\partial_R \df^1_\Sigma (R))  
    \end{aligned}
\end{equation}
Combining all these we get
\begin{equation}\label{var1st}
    \begin{aligned}
        \Tr  \langle   \hat{\Sigma}_{\theta_1} \rangle=&\alpha d \Delta_T \sigma_X^{-2} \frac{(\df^1_\Sigma(R))^2}{1-\alpha \df^2_{\Sigma}(R)}-2e^{2T}d(1- \ \df^1_\Sigma(R))\\
        &+e^{2T}d \left(2-2 \ \df^1_\Sigma(R)+\frac{\Rh}{1-\alpha \df^2_\Sigma(R) }\partial_R \df^1_\Sigma (R)\right)
    \end{aligned}
\end{equation}
Now we turn to evaluate $  \Tr  \langle   \hat{\Sigma}_{\theta_1}^2 \rangle$. We want to keep track of terms that are order $d$ and ignore sub-leading terms. This restricts possible contractions of $z, z^T$. We get a factor of $d$ only from contractions that happen next to each other. Keeping only those terms   
\begin{equation}
\begin{aligned}
    \Tr \langle \hat{\Sigma}_{\theta_1}^2\rangle \approx &\langle  \alpha^2 \Delta_T^2\Tr (\Sh^2 (\Sh+\Rh)^{-4})+e^{4T}\Rh^4\Tr (\Sh+\Rh)^{-4}+4e^{4T}\Rh^2\Tr (\Sh+\Rh)^{-2}\\
    &+2\alpha \Delta_Te^{2T}\Rh^2 \Tr (\Sh (\Sh+\Rh)^{-4})-4\alpha \Delta_Te^{2T}\Rh \Tr \Sh (\Sh+\Rh)^{-3}-4e^{4T}\Rh^3 \Tr  (\Sh+\Rh)^{-3}\\
     &+2\alpha \Delta_Te^{2T} \Tr (\Sh (\Sh+\Rh)^{-2})\rangle
\end{aligned}
\end{equation}
All these expectation values can be calculated from a generic term of the form for integer $a>0,b\geq0$
\begin{equation}\label{identity2}
    \begin{aligned}
      C_{a,b}= \langle \Tr (\Sh^a (\Sh+\Rh)^{-(a+b)}) \rangle&=\frac{(-1)^{b}}{a(a+1)\dots (a+b-1)}\partial_{\Rh}^b \langle  \Tr (\Sh^a (\Sh+\Rh)^{-a})\rangle \\
        &=d\frac{\Gamma(a)}{\Gamma(a+b)}(-\partial_{\Rh})^b \df^a_{\Sh}(\Rh)
    \end{aligned}
\end{equation}
Another identity that is useful is the following 
\begin{equation}\label{identity3}
    \begin{aligned}
    B_a&= \langle  \Tr  \Rh^a (\Sh+\Rh)^{-a} \rangle \\
    &=\langle \Tr( \Rh^{a-1} (\Sh+\Rh)^{-(a-1)}-\Rh^{a-1}\Sh (\Sh+\Rh)^{-a})  \rangle\\
    &=\langle \Tr(1-\sum_{i=1}^a\Rh^{a-i}\Sh (\Sh+\Rh)^{-(a-i+1)})\rangle\\
    &=d-\sum_{i=1}^a \Rh^{a-i}C_{1,a-i}
    \end{aligned}
\end{equation}
Now we turn to calculate expression for the symbols defined above. To get an explicit formula for $C_{a,b}$ first we replace the derivative with respect to $\hat{R}$ by a derivative with respect to $R$ with the chain rule given in the second equation on (\ref{identity1}). Next we use the recursion relation in the first equation on (\ref{identity1}) to express everything in terms of $\df^1_{\Sh}(\Rh)  \simeq   \df^1_{\Sigma}(R)$. Finally to perform the derivatives we use the self-consistency equation of the ridge parameter given in the last equation on (\ref{MDE}). Finally we use (\ref{formula1}). Once $C_{a,b}$s are computed we use the recursion relation to compute $B_a$s. 
The expression for these quantities for $\alpha>1$ are complicated. They are given as follows

\begin{equation}
    \begin{aligned}
        & C_{1,1}=\frac{2 d \sigma _X^2 \left(R+\sigma _X^2\right){}^2}{\left| -\left((\alpha -1) \sigma _X^4\right)+2 R \sigma _X^2+R^2\right|  \left(\left| -\left((\alpha -1) \sigma _X^4\right)+2 R \sigma _X^2+R^2\right| +R^2+2 R \sigma _X^2+(\alpha +1) \sigma _X^4\right)}\\
        & C_{1,2}=\frac{d \sigma _X^2 \left(R+\sigma _X^2\right){}^3}{\left| -\alpha  \sigma _X^4+\sigma _X^4+2 R \sigma _X^2+R^2\right| {}^3}\\
        & C_{1,3}=\frac{d \sigma _X^2 \left(R+\sigma _X^2\right){}^4 \left(R^2+2 R \sigma _X^2+(\alpha +1) \sigma _X^4\right) \text{sgn}\left(-\left((\alpha -1) \sigma _X^4\right)+2 R \sigma _X^2+R^2\right)}{\left(R^2+2 R \sigma _X^2-\left((\alpha -1) \sigma _X^4\right)\right){}^5}\\
        & C_{2,1}=\begin{array}{cc}
 \Biggl\{ & 
\begin{array}{cc}
 \frac{d \left(R+\sigma _X^2\right){}^3 \left(3 R^2 \sigma _X^2+R^3-3 (\alpha -1) R \sigma _X^4+(\alpha -1)^2 \sigma _X^6\right)}{\alpha  \sigma _X^2 \left(-R^2-2 R \sigma _X^2+(\alpha -1) \sigma _X^4\right){}^3} \hspace{2cm}\text{When } R \left(R+2 \sigma _X^2\right)<(\alpha -1) \sigma _X^4 \\
 \frac{d \sigma _X^4 \left((\alpha +1) R^3+3 R^2 \sigma _X^2-3 (\alpha -1) R \sigma _X^4+(\alpha -1)^2 \sigma _X^6\right)}{\left(R^2+2 R \sigma _X^2-\left((\alpha -1) \sigma _X^4\right)\right){}^3} \hspace{2cm} \text{Otherwise} \\
\end{array}
 \\ 
\end{array}\\
& C_{2,2}=\frac{d \sigma _X^4 \left(R+\sigma _X^2\right){}^4 \left((\alpha +1) R^2-2 (\alpha -1) R \sigma _X^2+(\alpha -1)^2 \sigma _X^4\right) \text{sgn}\left(-\left((\alpha -1) \sigma _X^4\right)+2 R \sigma _X^2+R^2\right)}{\left(R^2+2 R \sigma _X^2-\left((\alpha -1) \sigma _X^4\right)\right){}^5}\\
& C_{2,3}=\frac{1}{\left(R^2+2 R \sigma _X^2-\left((\alpha -1) \sigma _X^4\right)\right){}^7}(d \sigma _X^4 \left(R+\sigma _X^2\right){}^5 ((\alpha +1) R^4+3 ((\alpha -2) \alpha +2) R^2 \sigma _X^4-(\alpha -4) R^3 \sigma _X^2\\
& \hspace{5cm} +(\alpha -4) (\alpha -1) R \sigma _X^6+(\alpha -1)^2 (\alpha +1) \sigma _X^8) \text{sgn}\left(-\left((\alpha -1) \sigma _X^4\right)+2 R \sigma _X^2+R^2\right))\\
& C_{3,1}=\begin{array}{cc}
\Biggl \{ & 
\begin{array}{cc}
 \frac{1}{\alpha  \sigma _X^2 \left(-R^2-2 R \sigma _X^2+(\alpha -1) \sigma _X^4\right){}^5}(d \left(R+\sigma _X^2\right){}^4 ((\alpha -1)^2 (\alpha +15) R^2 \sigma _X^8-20 (\alpha -1) R^3 \sigma _X^6\\
 -5 (\alpha -3) R^4 \sigma _X^4
 +6 R^5 \sigma _X^2+R^6-6 (\alpha -1)^3 R \sigma _X^{10}\\
 +(\alpha -1)^4 \sigma _X^{12})) \hspace{1cm}\text{When } R \left(R+2 \sigma _X^2\right)<(\alpha -1) \sigma _X^4 \\
 \frac{1}{\left(R^2+2 R \sigma _X^2-\left((\alpha -1) \sigma _X^4\right)\right){}^5}(d \sigma _X^6 ((\alpha  (\alpha +3)+1) R^6+(\alpha -1)^2 (\alpha +15) R^2 \sigma _X^8\\+4 (\alpha -1) ((\alpha -1) \alpha -5) R^3 \sigma _X^6\\
 +(\alpha  ((\alpha -12) \alpha +6)+15) R^4 \sigma _X^4-2 ((\alpha -5) \alpha -3) R^5 \sigma _X^2-6 (\alpha -1)^3 R \sigma _X^{10}+(\alpha -1)^4 \sigma _X^{12})) \hspace{2cm}\text{Otherwise} \\
\end{array}
 \\
\end{array}\\
& C_{3,2}=\frac{1}{\left(R^2+2 R \sigma _X^2-\left((\alpha -1) \sigma _X^4\right)\right){}^7}(d \sigma _X^6 \left(R+\sigma _X^2\right){}^5 ((\alpha  (\alpha +3)+1) R^4+3 \left(\alpha ^3-3 \alpha +2\right) R^2 \sigma _X^4\\&+2 \left(-3 \alpha ^2+\alpha +2\right) R^3 \sigma _X^2
-4 (\alpha -1)^3 R \sigma _X^6+(\alpha -1)^4 \sigma _X^8) \text{sgn}\left(-\alpha  \sigma _X^4+\sigma _X^4+2 R \sigma _X^2+R^2\right))\\
& C_{3,3}=\frac{1}{\left(R^2+2 R \sigma _X^2-\left((\alpha -1) \sigma _X^4\right)\right){}^9}(d \sigma _X^6 \left(R+\sigma _X^2\right){}^6 ((\alpha  (\alpha +3)+1) R^6+6 \left(-\alpha ^2+\alpha +1\right) R^5 \sigma _X^2
\\
& \hspace{1cm}+3 (\alpha -1)^2 (\alpha  (2 \alpha -3)+5) R^2 \sigma _X^8-4 (\alpha -1) (2 (\alpha -2) \alpha +5) R^3 \sigma _X^6
+3 (\alpha  (2 (\alpha -1) \alpha -3)+5) R^4 \sigma _X^4\\
& \hspace{4cm} -6 (\alpha -1)^3 R \sigma _X^{10}+(\alpha -1)^4 (\alpha +1) \sigma _X^{12}) \text{sgn}\left(-\alpha  \sigma _X^4+\sigma _X^4+2 R \sigma _X^2+R^2\right))
    \end{aligned}
\end{equation}
\begin{equation}
    \begin{aligned}
       & B_1=\begin{array}{cc}
 \Biggl\{ & 
\begin{array}{cc}
 \frac{d R}{R+\sigma _X^2} &\text{When } R \left(R+2 \sigma _X^2\right)\geq (\alpha -1) \sigma _X^4 \\
 d-\frac{d \left(\frac{R}{\sigma _X^2}+1\right)}{\alpha } & \text{Otherwise} \\
\end{array}
 \\
\end{array}\\
& B_2=\begin{array}{cc}
\Biggl \{ & 
\begin{array}{cc}
 \frac{d R^2}{R^2+2 R \sigma _X^2-\alpha  \sigma _X^4+\sigma _X^4} & \text{When } R \left(R+2 \sigma _X^2\right)\geq (\alpha -1) \sigma _X^4 \\
 d \left(-\frac{1}{\alpha }-\frac{R^2}{R^2+2 R \sigma _X^2-\alpha  \sigma _X^4+\sigma _X^4}+1\right) & \text{Otherwise} \\
\end{array}
 \\
\end{array}\\
& B_3=\begin{array}{cc}
 \Biggl\{ & 
\begin{array}{cc}
 \frac{d R^3 \left(3 R^2 \sigma _X^2+R^3+3 R \sigma _X^4-\left(\left(\alpha ^2-1\right) \sigma _X^6\right)\right)}{\left(R^2+2 R \sigma _X^2-\left((\alpha -1) \sigma _X^4\right)\right){}^3} & \text{When } R\left(R+2 \sigma _X^2\right)\geq (\alpha -1) \sigma _X^4 \\
 \frac{d \left(R-(\alpha -1) \sigma _X^2\right){}^3 \left(3 R^2 \sigma _X^2+R^3+3 R \sigma _X^4-\left((\alpha -1) \sigma _X^6\right)\right)}{\alpha  \left(-R^2-2 R \sigma _X^2+(\alpha -1) \sigma _X^4\right){}^3} & \text{Otherwise} \\
\end{array}
 \\
\end{array}\\
& B_4=\begin{array}{cc}
\Biggl \{ & 
\begin{array}{cc}
 \frac{1}{\left(R^2+2 R \sigma _X^2-\left((\alpha -1) \sigma _X^4\right)\right){}^5}(d R^4 (4 \left(-\alpha ^2+\alpha +5\right) R^3 \sigma _X^6
 +(\alpha  ((\alpha -12) \alpha +6)+15) R^2 \sigma _X^8\\+(\alpha +15) R^4 \sigma _X^4+6 R^5 \sigma _X^2
 +R^6+2 (\alpha  ((\alpha -6) \alpha +2)+3) R \sigma _X^{10}\\+(\alpha -1)^2 (\alpha  (\alpha +3)+1) \sigma _X^{12})) \hspace{2cm} \text{When } R \left(R+2 \sigma _X^2\right)\geq (\alpha -1) \sigma _X^4 \\
 \frac{d \left(R-(\alpha -1) \sigma _X^2\right){}^4 \left(-5 (\alpha -3) R^2 \sigma _X^8+(\alpha +15) R^4 \sigma _X^4+20 R^3 \sigma _X^6+6 R^5 \sigma _X^2+R^6-6 (\alpha -1) R \sigma _X^{10}+(\alpha -1)^2 \sigma _X^{12}\right)}{\alpha  \left(-R^2-2 R \sigma _X^2+(\alpha -1) \sigma _X^4\right){}^5} \hspace{1cm} \text{Otherwise} \\
\end{array}
 \\
\end{array}
    \end{aligned}
\end{equation}
Explicit expression of some of these symbols that will be required later is given below for $\alpha<1$.
\begin{equation}
    \begin{aligned}
       & B_1=\frac{d R}{R+\sigma _X^2} \\
       & B_2=\frac{d R^2}{R^2+2 R \sigma _X^2-\alpha  \sigma _X^4+\sigma _X^4}\\
     & B_3= \frac{d R^3 \left(3 R^2 \sigma _X^2+R^3+3 R \sigma _X^4-\left(\left(\alpha ^2-1\right) \sigma _X^6\right)\right)}{\left(R^2+2 R \sigma _X^2-\left((\alpha -1) \sigma _X^4\right)\right){}^3}\\
       & B_4=\frac{1}{\left(R^2+2 R \sigma _X^2-\left((\alpha -1) \sigma _X^4\right)\right){}^5} (d R^4 (4 \left(-\alpha ^2+\alpha +5\right) R^3 \sigma _X^6+(\alpha  ((\alpha -12) \alpha +6)+15) R^2 \sigma _X^8\\
       &\hspace{1cm}+(\alpha +15) R^4 \sigma _X^4+6 R^5 \sigma _X^2+R^6+2 (\alpha  ((\alpha -6) \alpha +2)+3) R \sigma _X^{10}+(\alpha -1)^2 (\alpha  (\alpha +3)+1) \sigma _X^{12}))
    \end{aligned}
\end{equation}

\begin{equation}
    \begin{aligned}
        & C_{1,1}=\frac{d \sigma _X^2}{R^2+2 R \sigma _X^2-\alpha  \sigma _X^4+\sigma _X^4}\\
        &  C_{1,2}=\frac{d \sigma _X^2 \left(R+\sigma _X^2\right){}^3}{\left(R^2+2 R \sigma _X^2-\left((\alpha -1) \sigma _X^4\right)\right){}^3}\\
        & C_{1,3}=\frac{d \sigma _X^2 \left(R+\sigma _X^2\right){}^4 \left(R^2+2 R \sigma _X^2+(\alpha +1) \sigma _X^4\right)}{\left(R^2+2 R \sigma _X^2-\left((\alpha -1) \sigma _X^4\right)\right){}^5}\\
        & C_{2,1}=\frac{d \sigma _X^4 \left((\alpha +1) R^3+3 R^2 \sigma _X^2-3 (\alpha -1) R \sigma _X^4+(\alpha -1)^2 \sigma _X^6\right)}{\left(R^2+2 R \sigma _X^2-\left((\alpha -1) \sigma _X^4\right)\right){}^3}\\
        & C_{2,2}=\frac{d \sigma _X^4 \left(R+\sigma _X^2\right){}^4 \left((\alpha +1) R^2-2 (\alpha -1) R \sigma _X^2+(\alpha -1)^2 \sigma _X^4\right)}{\left(R^2+2 R \sigma _X^2-\left((\alpha -1) \sigma _X^4\right)\right){}^5}\\
    \end{aligned}
\end{equation}

\begin{equation}
    \begin{aligned}
        & C_{2,3}=\frac{1}{\left(R^2+2 R \sigma _X^2-\left((\alpha -1) \sigma _X^4\right)\right){}^7}d \sigma _X^4 \left(R+\sigma _X^2\right){}^5 ((\alpha +1) R^4+3 ((\alpha -2) \alpha +2) R^2 \sigma _X^4\\
        &\hspace{8cm}-(\alpha -4) R^3 \sigma _X^2+(\alpha -4) (\alpha -1) R \sigma _X^6+(\alpha -1)^2 (\alpha +1) \sigma _X^8)\\
        & C_{3,1}=\frac{1}{\left(R^2+2 R \sigma _X^2-\left((\alpha -1) \sigma _X^4\right)\right){}^5}(d \sigma _X^6 ((\alpha  (\alpha +3)+1) R^6+(\alpha -1)^2 (\alpha +15) R^2 \sigma _X^8\\
        &+4 (\alpha -1) ((\alpha -1) \alpha -5) R^3 \sigma _X^6
         +(\alpha  ((\alpha -12) \alpha +6)+15) R^4 \sigma _X^4-2 ((\alpha -5) \alpha -3) R^5 \sigma _X^2\\
         &\hspace{10cm}-6 (\alpha -1)^3 R \sigma _X^{10}+(\alpha -1)^4 \sigma _X^{12}))\\
        & C_{3,2}=\frac{1}{\left(R^2+2 R \sigma _X^2-\left((\alpha -1) \sigma _X^4\right)\right){}^7}d \sigma _X^6 \left(R+\sigma _X^2\right){}^5 ((\alpha  (\alpha +3)+1) R^4+3 \left(\alpha ^3-3 \alpha +2\right) R^2 \sigma _X^4\\
        &\hspace{6cm}+2 \left(-3 \alpha ^2+\alpha +2\right) R^3 \sigma _X^2-4 (\alpha -1)^3 R \sigma _X^6+(\alpha -1)^4 \sigma _X^8)\\
        & C_{3,3}=\frac{1}{\left(R^2+2 R \sigma _X^2-\left((\alpha -1) \sigma _X^4\right)\right){}^9}(d \sigma _X^6 \left(R+\sigma _X^2\right){}^6 ((\alpha  (\alpha +3)+1) R^6
        +6 \left(-\alpha ^2+\alpha +1\right) R^5 \sigma _X^2\\
        & \hspace{6cm}+3 (\alpha -1)^2 (\alpha  (2 \alpha -3)+5) R^2 \sigma _X^8
        -4 (\alpha -1) (2 (\alpha -2) \alpha +5) R^3 \sigma _X^6\\
        &\hspace{6cm}+3 (\alpha  (2 (\alpha -1) \alpha -3)+5) R^4 \sigma _X^4-6 (\alpha -1)^3 R \sigma _X^{10}+(\alpha -1)^4 (\alpha +1) \sigma _X^{12}))
    \end{aligned}
\end{equation}

In terms of these symbols we have the following explicit formula
\begin{equation}\label{identity4}
    \begin{aligned}
     \Tr \langle \hat{\Sigma}_{\theta_1}\rangle =&\alpha \Delta_T C_{1,1}+e^{2T}(B_2-2B_1)\\
        \Tr \langle \hat{\Sigma}_{\theta_1}^2\rangle = & \alpha^2 \Delta_T^2C_{2,2}+e^{4T}B_4+4e^{4T}B_2\\
    &+2\alpha \Delta_Te^{2T}\Rh^2 C_{1,3}-4\alpha \Delta_Te^{2T}\Rh C_{1,2}-4e^{4T}B_3
     +2\alpha \Delta_Te^{2T} C_{1,1}
    \end{aligned}
\end{equation}
As a summary our final expression for variance term in KL divergence is given by (\ref{KL2}) along with (\ref{identity1}),(\ref{identity2}),(\ref{identity3}) and (\ref{identity4}). 
Since the expression is fairly complicated we won't present explicit formula for it. To understand the implications of the formula we look at ridgeless limit $\hat{R}\to 0$. 

Putting all the results together, for $\alpha<1$, the ridgeless formula takes the following form
\begin{equation}\label{KL2f}
    \begin{aligned}
     \langle \text{KL}(\rho_G|\rho)_{\text{var}} \rangle
     =& \frac{1}{4}(e^{-2T} +\sigma^{-2}\Delta_T)^2 (2 e^{2T} \frac{\alpha}{1-\alpha}\Delta_T  d \ \sigma_X^{-2}+ \frac{\alpha^2}{(1-\alpha)^3}\Delta_T^2   d \ \sigma_X^{-4}) \\
     &+\frac{1}{2}(e^{-2T} +\sigma^{-2}\Delta_T)(\sigma^{-2}\Delta_Te^{2T})( \frac{\alpha}{1-\alpha} \Delta_T d \  \sigma_X^{-2} )+\frac{d}{4}(\sigma^{-2}e^{2T}\Delta_T)^2\\
     =&\frac{d\alpha  \hat{\lambda } e^{-4 T} \left(e^{2 T}-1\right)}{2 (1-\alpha)}+ \frac{d\hat{\lambda }^2 e^{-8 T} \left(e^{2 T}-1\right)^2 \left(\alpha ^2+(1-\alpha)^3 e^{4 T}+4 \alpha  (1-\alpha)^2 e^{2 T}\right)}{4 (1-\alpha)^3}
    \end{aligned}
\end{equation}
If we further consider $\alpha \to 0$ approximation we see that $\langle \text{KL}(\rho_G|\rho)_{\text{var}} \rangle \propto d \lambda^2$ to the leading order. Also note that $\langle \text{KL}(\rho_G|\rho)_{\text{var}} \rangle/d$ is an increasing function of $\alpha$ in this regime.

Ridgeless limit for $\alpha>1$ is more involved and it is given by
\begin{equation}\label{KL2fm}
    \begin{aligned}
     \langle \text{KL}(\rho_G|\rho)_{\text{var}} \rangle
     =& \frac{1}{4}(e^{-2T} +\sigma^{-2}\Delta_T)^2 (2 e^{2T} \frac{1}{(\alpha-1)}\Delta_T  d \ \sigma_X^{-2}+ \frac{\alpha^2}{(\alpha-1)^3}\Delta_T^2   d \ \sigma_X^{-4}+e^{4T}d\bigg(1-\frac{1}{\alpha}\bigg)) \\
     &+\frac{1}{2}(e^{-2T} +\sigma^{-2}\Delta_T)(\sigma^{-2}\Delta_Te^{2T})( \frac{1}{\alpha-1} \Delta_T d \  \sigma_X^{-2} -e^{2T}d\bigg(1-\frac{1}{\alpha}\bigg))+\frac{d}{4}(\sigma^{-2}e^{2T}\Delta_T)^2\\
     =&  d\frac{\alpha -1}{4 \alpha }+ \frac{d\hat{\lambda } e^{-4 T} \left(e^{2 T}-1\right)}{2 (\alpha -1)}+\frac{d\hat{\lambda }^2 e^{-8 T} \left(e^{2 T}-1\right)^2 \left(\alpha ^3+(\alpha -1)^3 e^{4 T}+4 \alpha  (\alpha -1)^2 e^{2 T}\right)}{4 (\alpha -1)^3 \alpha }
    \end{aligned}
\end{equation}
In this domain $\langle \text{KL}(\rho_G|\rho)_{\text{var}} \rangle/d$ is no longer a monotonic function of $\alpha$.
It is easy to see from the expression above that as $\alpha \to 1$ both from $\alpha>1$ and $\alpha<1$ side, KL divergence becomes unbounded.

\end{document}